\documentclass{article}

\usepackage[preprint]{neurips_2019}
\setcitestyle{square,comma,numbers}

\usepackage[utf8]{inputenc} 
\usepackage[T1]{fontenc}    
\usepackage[hidelinks]{hyperref}       
\usepackage{url}            
\usepackage{booktabs}       
\usepackage{amsfonts}       
\usepackage{nicefrac}       
\usepackage{microtype}      
\usepackage[page]{appendix}

\usepackage{amsmath, amsfonts, amsthm}
\usepackage{thmtools, thm-restate}

\newtheorem{definition}{Definition}
\newtheorem{lemma}{Lemma}
\newtheorem{theorem}{Theorem}
\newtheorem{corollary}{Corollary}
\newcommand{\reals}{\mathbb{R}}
\newcommand{\expectation}{\mathbb{E}}
\newcommand{\ylinkSimple}[1][\Phi]{Y^{#1}}
\newcommand{\ylink}[2][\Phi]{Y_{#2}^{#1}}
\newcommand{\ylinkEstSimple}[1][\Phi]{\Tilde{Y}^{#1}}
\newcommand{\ylinkEst}[2][\Phi]{\Tilde{Y}_{#2}^{#1}}
\newcommand{\algSimple}{L}
\newcommand{\alg}[1][t]{L_{#1}(h)}
\newcommand{\norm}[1]{\| #1 \|}
\newcommand{\supReward}{U}

\makeatletter
    \@ifdefinable{\epsGenBlackwellCond}{\def\epsGenBlackwellCond/{$(\Phi, f, \epsilon)$-Blackwell condition}}
\makeatother
\makeatletter
    \@ifdefinable{\epsGenBlackwellCondTitle}{\def\epsGenBlackwellCondTitle/{$(\Phi, f, \epsilon)$-Blackwell Condition}}
\makeatother

\title{Bounds for Approximate Regret-Matching Algorithms}

\author{
    Ryan D'Orazio,
    Dustin Morrill,
    James R. Wright \\
    Department of Computing Science \\
    University of Alberta \\
    \{rdorazio, morrill, james.wright\}@ualberta.ca
}

\begin{document}

\maketitle

\begin{abstract}
  A dominant approach to
  solving large imperfect-information games is
  \emph{Counterfactural Regret Minimization (CFR)}.
  In CFR,
  many regret minimization problems are combined to
  solve the game.
  For very large games, abstraction is typically
  needed to render CFR tractable.
  Abstractions are often manually tuned,
  possibly removing important strategic differences
  in the full game and harming performance.
  Function approximation provides a natural solution to finding
  good abstractions to approximate the full game.
  A common approach to incorporating function approximation
  is to learn the inputs needed for a regret minimizing algorithm, allowing
  for generalization across many regret minimization problems.
  This paper gives regret bounds when a regret minimizing algorithm uses estimates
  instead of true values.
  This form of analysis is the first to generalize to a larger class
  of $(\Phi, f)$-regret matching algorithms, and includes different forms of regret
  such as swap, internal, and external regret.
  We demonstrate how these results
  give a slightly tighter bound for Regression Regret-Matching (RRM), and present a
  novel bound for combining regression with Hedge.
\end{abstract}

\section{Introduction}
The dominant framework for approximating Nash equilibria
in extensive-form games with imperfect information is Counterfactual Regret Minimization (CFR), and it has successfully been used to solve and expertly play human-scale poker games~\cite{bowling2015heads,moravvcik2017deepstack,brown2018superhuman,brown2019superhuman}. This framework is built on the idea of decomposing a game into a network of simple regret minimizers~\cite{zinkevich2008regret,farina2019regret}.
For very large games, abstraction is typically used to reduce the size and yield a strategically similar game that is feasible to solve with CFR~\citep{zinkevich2008regret, waugh2009abstraction, johanson2013evaluating, ganzfried2013action}.

Function approximation is a natural generalization of abstraction.
In CFR, this amounts to estimating the regrets for each regret minimizer instead of storing them all in a table~\citep{waugh2015solving,morrill2016,deepCFR,neuRD,Exp_Descent,steinberger2019single}.
Function approximation can be competitive with domain specific state abstraction~\cite{waugh2015solving,morrill2016,deepCFR}, and
in some cases is able to outperform tabular CFR without abstraction if the players
are optimizing against their best responses~\cite{Exp_Descent}.

Combining regression and regret-minimization with applications to CFR was initially studied by
Waugh et. al.~\cite{waugh2015solving}, introducing the RRM theorem---giving a sufficient
condition for function approximator error to still achieve no external regret.
In this paper we generalize the RRM theorem to a larger class of regret-minimizers and $\Phi$-regret---a set of regret metrics that include
external regret, internal regret, and swap regret.
Extending to a larger class of regret-minimizers provides insight into the effectiveness of
combining function approximation and regret minimization---the effect of function approximation
error on the bounds will vary between algorithms.
Furthermore, extending to other algorithms can give theory for existing
or future methods.
For example, there has been interest in a functional version of hedge,
an algorithm within the studied class,
for general multiagent
and non-stationary settings that can outperform softmax policy gradient methods~\cite{neuRD}.
Extending to a more general class of regret metrics such as internal regret allows for
potentially-novel applications of regret minimization and function approximation
including finding an approximate correlated equilibrium~\cite{cesa2006prediction}.

\section{Preliminaries}
We adopt the notation from Greenwald et al. \cite{greenwald2006bounds} to describe an
online decision problem (ODP).
An ODP consists of a set of possible actions $A$ and set of possible rewards $\mathcal{R}$.
In this paper we assume a finite set of actions and
bounded\footnote{The restriction of positive rewards is without loss of generality and is only used for convenience.} $\mathcal{R} \subset \reals_+$ where
$\sup \mathcal{R} = \supReward$.
The tuple $(A, \mathcal{R})$ fully characterizes the problem and is referred to as a reward
system. Furthermore, let $\Pi$ denote the set of reward functions $r: A \mapsto \mathcal{R}$.

At each round $t$ an agent selects a distribution over actions $q_t \in \Delta(A)\footnote{$\Delta(A)$ is the set of all probability distributions over actions in $A$.}$, samples an action $a_t \sim q_t$ and then
receives the reward function $r_t \in \Pi$.
The agent is able to compute the rewards for actions that were not taken at time
$t$ in contrast to the bandit setting where the agent only observes $r_t(a_t)$.
Crucially, each $r_t$ is allowed to be selected arbitrarily from $\Pi$. As a consequence, this ODP model is flexible enough to encompass multi-agent, adversarial interactions, and game theoretic equilibrium concepts even though it is described from the perspective of a single agent's decisions.

A learning algorithm in an ODP selects $q_t$ using information
from the history of observations and actions previously
taken. We denote this information at time $t$ as historyt $h \in H_t := A^t \times \Pi^t$, where $H_0 :=
\{\emptyset\}$. Formally, an online learning algorithm is a sequence of functions
$\{L_t \}_{t=1}^\infty$, where $L_t : H_{t-1} \mapsto \Delta(A)$.

\subsection{Action Transformations}

To generalize the analysis to different forms of regret (e.g. swap, internal, and external
regret), it is useful to define action transformations.
Action transformations are functions of the form $\phi : A \mapsto \Delta(A)$, giving
a distribution over actions for each action input. Let $\Phi_{ALL} := \Phi_{ALL}(A)$
denote the set of all action transformations for the set of actions $A$ and
$\Phi_{SWAP} := \Phi_{SWAP}(A)$ the set of all action transformations with codomain
as the set of all pure strategies for action set $A$.

Two important
subsets of $\Phi_{SWAP}$ are
$\Phi_{EXT}$ and $\Phi_{INT}$.
$\Phi_{EXT}$ denotes the set of all external transformations---the set of constant
action transformations in $\Phi_{SWAP}$. More formally, if $\delta_a \in \Delta(A)$ is the
distribution with full weight on action $a$, then
$\Phi_{EXT}  = \{\phi : \exists y \in A \, \forall x \in A \quad \phi(x) = \delta_y \}$.

$\Phi_{INT}$ consists of the set of all possible internal transformations for action
set $A$,
where an internal
transformation from action $a$ to action $b$ is defined as
$\phi_{INT}^{(a,b)}(x) = \delta_b$ if $x=a$, $\phi_{INT}^{(a,b)}(x) = \delta_x$ otherwise.

We have that $|\Phi_{SWAP}| = |A|^{|A|}, \Phi_{EXT} =|A|,
|\Phi_{INT}| = |A|^2 - |A| +1$ \cite{greenwald2006bounds}.

We will also make use of the linear transformation $[\phi] : \Delta(A) \mapsto \Delta(A)$
defined as $[\phi](q) = \sum_{a \in A}q(a)\phi(a)$.

\subsection{Regret}

For a given action transformation $\phi$ we can compute the difference in
expected reward for a particular action and reward function. This expected difference,
known as $\phi$-regret, is
denoted by $\rho^\phi(a,r) = \expectation_{s \sim \phi(a)}[r(s)] - r(a)$.
For a set of action transformations $\Phi$, the $\Phi$-regret vector is
$\rho^\Phi(a,r) = (\rho^\phi(a,r))_{\phi \in \Phi}$. Note the expected value of
$\phi$-regret if the agent chooses $q \in \Delta(A)$ is $\expectation_{a \sim q}[\rho^\phi (a,r)]$.

For an ODP with observed history $h$ at time $t$,
with reward functions $\{ r_s\}_{s=1}^t$ and
actions $\{a_s\}_{s=1}^t$,
the cumulative $\Phi$-regret for time $t$
and action transformations $\Phi$ is $R^\Phi_t(h) = \sum_{s=1}^t{\rho^\Phi(a_s,r_s)}$.
For brevity we will omit the $h$ argument, and for convenience we set $R^\Phi_0 = 0$. Note that  $R^\Phi_t$ is a
random vector, and we seek to bound
\begin{align}
    \expectation\left[ \frac{1}{t} \underset{\phi \in \Phi}{\mbox{max }} R^\phi_t \right]. \label{objective}
\end{align}

Choosing $\Phi_{EXT},\Phi_{INT}, \Phi_{SWAP}$ for (\ref{objective})
amounts to minimizing external regret,
internal regret, and swap regret respectively.
One can also change (\ref{objective}) by interchanging the max and the expectation.
In RRM, $ \underset{\phi \in \Phi}{\mbox{max }}\expectation\left[ \frac{1}{t}  R^\phi_t \right]$ is bounded~\citep{waugh2015solving, morrill2016},
however, bounds for (\ref{objective}) still apply~\cite[Corollary 18]{greenwald2006bounds}.

\section{Approximate Regret-Matching}
Given a set of action transformations $\Phi$ and a link function 
$f : \reals^{|\Phi|} \mapsto \reals^{|\Phi|}_+ $
that is subgradient to a convex potential function $G: \reals^{|\Phi|} \mapsto \reals$,
where $\reals^N_+$ denotes the $N$-dimensional positive orthant\footnote{
Note that as long as $G$ is bounded from above on the negative orthant then the codomain of $f$ is the positive orthant.}, we can define a general class of online learning
algorithms known as \emph{$(\Phi, f)$-regret-matching} algorithms~\cite{greenwald2006bounds}.
A \emph{$(\Phi, f)$-regret-matching} algorithm at time $t$ 
chooses $q \in \Delta(A)$ that is  
a fixed point\footnote{Note that since $M_t$ is a linear 
operator over the simplex $\Delta(A)$, the fixed point always exists by the Brouwer fixed point theorem.} of 
$M_t = \nicefrac{\sum_{\phi \in \Phi}\ylink[\phi]{t}[\phi]}{\sum_{\phi \in \Phi}\ylink[\phi]{t}}$
where $\ylink{t} = (\ylink[\phi]{t})_{\phi \in \Phi} = f(R^\Phi_{t-1})$
when $R^\Phi_{t-1} \in \reals^{|\Phi|}_+ \setminus \{0\}$ and arbitrarily otherwise.
If $\Phi = \Phi_{EXT}$ then the fixed point of $M_t$ is a distribution
$q \propto \ylink{t}$~\cite{greenwaldtech}.
Examples  of $(\Phi, f)$-regret-matching algorithms
include Hart's algorithm~\cite{Hart00Rm}--typically called ``regret-matching'' or the polynomial weighted average forecaster~\cite{cesa2006prediction}--and Hedge~\cite{freund1997decision}--the exponentially weighted average forecaster~\cite{cesa2006prediction},
with link functions $f(x)_i = (x_i^+)^{p-1}$ 
for $p\geq 1$, and
$f(x)_i = e^{\eta x_i}$ with parameter $\eta > 0$, respectively.

A useful technique to bounding regret when estimates are used in place of true values
is to define
an $\epsilon-$Blackwell condition, as was used in the RRM theorem~\citep{waugh2015solving}.
The analysis in RRM was specific to $\Phi = \Phi_{EXT}$ and the 
polynomial link $f$ with $p=2$.
To generalize across different link functions and $\Phi \subseteq \Phi_{ALL}$ we define the 
\epsGenBlackwellCond/.

\begin{definition}[\epsGenBlackwellCondTitle/] For a given reward system
$(A, \mathcal{R})$, finite set of action transformations $\Phi \subseteq \Phi_{ALL}$,
and link function $f: \reals^{|\Phi|} \mapsto \reals^{|\Phi|}_+$,
a learning algorithm satisfies the \epsGenBlackwellCond/ with value $\epsilon$ if
$
        f(R^\Phi_{t-1}(h)) \cdot \expectation_{a \sim \alg}[\rho^\Phi(a,r)] \leq \epsilon.
$
\end{definition}

The Regret Matching Theorem~\cite{greenwald2006bounds} shows that the
$(\Phi, f)$-Blackwell condition ($\epsilon =0$) holds with 
equality for \emph{$(\Phi, f)$-regret-matching} algorithms 
for any finite set of action transformations $\Phi$ and link function $f$. 

We seek to bound objective (\ref{objective})
when an algorithm at time $t$ chooses the fixed point of
$\Tilde{M}_t = \nicefrac{\sum_{\phi \in \Phi}{\ylinkEst[\phi]{t} [\phi]}}{\sum_{\phi \in \Phi}{\ylinkEst[\phi]{t}}}$,
when $\Tilde{R}^\Phi_{t-1} \in \reals^{|\Phi|}_+ \setminus \{0\}$ and arbitrarily otherwise,
where $\ylinkEst{t} = f(\Tilde{R}^\Phi_{t-1})$ and $\Tilde{R}^\Phi_{t-1}$ is an
estimate of $R^\Phi_{t-1}$, possibly from a function approximator.
Such an algorithm is referred to as approximate $(\Phi,f)$-regret-matching.

Similarly to the RRM theorem~\citep{waugh2015solving, morrill2016}, we show that the $\epsilon$ parameter of the \epsGenBlackwellCond/ depends on the error in approximating the exact link outputs, $\norm{\ylink{t}- \ylinkEst{t}}_1$.

\begin{restatable}{theorem}{epsbound}\label{thm-epsbound}
Given reward system (A,$\mathcal{R}$), a finite set of action
transformations $\Phi \subseteq \Phi_{ALL}$, and link function
$f: \reals^{|\Phi|} \mapsto \reals^{|\Phi|}_+$, then an
\emph{approximate ($\Phi,f$)-regret-matching} algorithm,
$\{L_t\}_{t=1}^\infty$, satisfies the \epsGenBlackwellCondTitle/ with $\epsilon \leq
2 \supReward \norm{\ylink{t} - \ylinkEst{t}}_1$, where $\ylink{t} = f(R^\Phi_{t-1})$, and  
$\ylinkEst{t} = f(\Tilde{R}^\Phi_{t-1})$.
\end{restatable}\textbf{All proofs are deferred to the appendix.}

For a $(\Phi,f)$-regret-matching algorithm, an approach to bounding
(\ref{objective}) is to use the $(\Phi, f)$-Blackwell condition and
provide a bound on
$\expectation[G(R^\Phi_t)]$ for an appropriate potential function $G$~\citep{greenwald2006bounds,cesa2006prediction}.
Bounding the regret (\ref{objective}) for an approximate $(\Phi, f)$-regret-matching algorithm
will be done similarly, except the
bound on $\epsilon$ from Theorem \ref{thm-epsbound} will be used.
Proceeding in this fashion yields the following theorem:
\begin{restatable}{theorem}{expectationbound}
Given a real-valued reward system $(A, \mathcal{R})$ a finite set
    $\Phi \subseteq \Phi_{ALL}$ of action transformations. If
    $\langle G, g, \gamma \rangle$ is a Gordon triple\footnote{See
    definition 2 in appendix.}, then an
    approximate $(\Phi, g)$-regret-matching algorithm $\{L_t \}_{t=1}^\infty$
    guarantees at all times $t \geq 0$
    \[
        \mathbb{E}[G(R^\Phi_t)] \leq G(0) +
        t \underset{a \in A, r \in \Pi}{\sup} \gamma(\rho^{\Phi}(a,r)) +
        2\supReward \sum_{s=1}^t{\norm{g(R^\Phi_{s-1})-g(\Tilde{R}^\Phi_{s-1})}_1}.
    \]
\end{restatable}

\section{Bounds}

\subsection{Polynomial Link}
Given the polynomial link function $f(x)_i = (x_i^+)^{p-1}$ we consider two cases $2 < p < \infty$
and $1 < p \leq 2$.
For the following results it is useful to denote
the maximal activation
$\mu(\Phi) = \mbox{max}_{a \in A}|\{\phi \in \Phi : \phi(a)\neq \delta_a \}|$ \cite{greenwald2006bounds}.

For the case $p>2$ we have the following bound on (\ref{objective}).
\begin{restatable}{theorem}{polyONE}
Given an ODP, a finite set of action transformations
$\Phi \subseteq \Phi_{ALL}$, and the polynomial link
function $f$ with $p>2$, then an approximate $(\Phi,f)$-
regret-matching algorithm guarantees
\[
      \expectation\left[\underset{\phi \in \Phi}{\normalfont \mbox{max}}\frac{1}{t} R^\phi_t \right]
       \leq \frac{1}{t}\sqrt{t(p-1)U^2(\mu(\Phi))^{2/p} +
        2 U \sum_{k=1}^t \norm{g(R^\Phi_{k-1}) - g(\Tilde{R}^\Phi_{k-1})}_1 }
\]
where $g: \reals^{|\Phi|} \mapsto \reals^{|\Phi|}_+$ and $g(x)_i = 0$ if $x_i \leq 0$ otherwise $g(x)_i= \frac{2(x_i)^{p-1}}{\norm{x^+}^{p-2}_p}$.
\end{restatable}

Similarly for the case $1 < p \leq 2$
we have the following.
\begin{restatable}{theorem}{polyTWO}
Given an ODP, a finite set of action transformations
$\Phi \subseteq \Phi_{ALL}$, and the polynomial link
function $f$ with $1 < p \leq2$, then an approximate $(\Phi,f)$-
regret-matching algorithm guarantees
        \[
      \expectation\left[\underset{\phi \in \Phi}{\normalfont \mbox{max}}\frac{1}{t} R^\phi_t \right]
       \leq \frac{1}{t}\left(tU^p\mu(\Phi) +
        2 U \sum_{k=1}^t \norm{g(R^\Phi_{k-1}) - g(\Tilde{R}^\Phi_{k-1})}_1 \right)^{1/p}
    \]
where $g: \reals^{|\Phi|} \mapsto \reals^{|\Phi|}_+$ and $g(x)_i= p(x^+_i)^{p-1}$.
\end{restatable}

In comparison to the RRM theorem~\citep{morrill2016}, the above bound is
tighter as there is no $\sqrt{|A|}$ term in front of the errors and the $|A|$ term
has been replaced by\footnote{For $\Phi =\Phi_{EXT}, \mu(\Phi)=|A|-1$.} $|A|-1$.
These improvements are due
to the tighter bound in Theorem \ref{thm-epsbound} and
the original $\Phi$-regret analysis~\cite{greenwald2006bounds}, respectively. Aside from these differences, the bounds
coincide.

\subsection{Exponential Link}
\begin{restatable}{theorem}{exponential}\label{thm-explink}
Given an ODP, a finite set of action transformations
$\Phi \subseteq \Phi_{ALL}$, and an exponential link
function $f(x)_i = e^{\eta x_i}$ with $\eta > 0$, then an approximate $(\Phi,f)$-
regret-matching algorithm guarantees
       \[
      \expectation\left[\underset{\phi \in \Phi}{\normalfont \mbox{max}}\frac{1}{t} R^\phi_t \right]
       \leq \frac{1}{t}\left(
       \frac{\normalfont \mbox{ln}|\Phi|}{\eta} + 2 U \sum_{k=1}^t \norm{g(R^\Phi_{k-1}) - g(\Tilde{R}^\Phi_{k-1})}_1
       \right)+ \frac{\eta U^2}{2}
    \]
where $g: \reals^{|\Phi|} \mapsto \reals^{|\Phi|}_+$ and $g(x)_i= e^{\eta x_i}/\sum_j{e^{\eta x_j}}$.
\end{restatable}

The Hedge algorithm corresponds to the exponential link function $f(x)_i = e^{\eta x_i}$ when $\Phi = \Phi_{EXT}$, so Theorem \ref{thm-explink} provides a bound on a regression Hedge algorithm. Note that in this case, the approximation error term is not inside a root function as it is under the polynomial link function. This seems to imply that at the level of link outputs, polynomial link functions have a better dependence on the approximation errors. However, $g$ in the exponential link function bound is normalized to the simplex while the polynomial link functions can take on larger values. So which link function has a better dependence on the approximation errors depends on the magnitude of the cumulative regrets, which depends on the environment and the algorithm's empirical performance.

\section*{Acknowledgments}

We acknowledge the support of the Natural Sciences and Engineering
Research Council of Canada (NSERC), the Alberta Machine Intelligence Institute (Amii), and Alberta
Treasury Branch (ATB).

\bibliographystyle{unsrt}

\bibliography{ref}

\begin{thebibliography}{10}

\bibitem{bowling2015heads}
Michael Bowling, Neil Burch, Michael Johanson, and Oskari Tammelin.
\newblock Heads-up limit hold'em poker is solved.
\newblock {\em Science}, 347(6218):145--149, 2015.

\bibitem{moravvcik2017deepstack}
Matej Morav{\v{c}}{\'{i}}k, Martin Schmid, Neil Burch, Viliam Lis{\`y}, Dustin
  Morrill, Nolan Bard, Trevor Davis, Kevin Waugh, Michael Johanson, and Michael
  Bowling.
\newblock Deepstack: Expert-level artificial intelligence in heads-up no-limit
  poker.
\newblock {\em Science}, 356(6337):508--513, 2017.

\bibitem{brown2018superhuman}
Noam Brown and Tuomas Sandholm.
\newblock Superhuman ai for heads-up no-limit poker: Libratus beats top
  professionals.
\newblock {\em Science}, 359(6374):418--424, 2018.

\bibitem{brown2019superhuman}
Noam Brown and Tuomas Sandholm.
\newblock Superhuman ai for multiplayer poker.
\newblock {\em Science}, 365(6456):885--890, 2019.

\bibitem{zinkevich2008regret}
Martin Zinkevich, Michael Johanson, Michael Bowling, and Carmelo Piccione.
\newblock Regret minimization in games with incomplete information.
\newblock In {\em Advances in neural information processing systems}, pages
  1729--1736, 2008.

\bibitem{farina2019regret}
Gabriele Farina, Christian Kroer, and Tuomas Sandholm.
\newblock Regret circuits: Composability of regret minimizers.
\newblock In {\em International Conference on Machine Learning}, pages
  1863--1872, 2019.

\bibitem{waugh2009abstraction}
Kevin Waugh, David Schnizlein, Michael Bowling, and Duane Szafron.
\newblock Abstraction pathologies in extensive games.
\newblock In {\em Proceedings of The 8th International Conference on Autonomous
  Agents and Multiagent Systems-Volume 2}, pages 781--788. International
  Foundation for Autonomous Agents and Multiagent Systems, 2009.

\bibitem{johanson2013evaluating}
Michael Johanson, Neil Burch, Richard Valenzano, and Michael Bowling.
\newblock Evaluating state-space abstractions in extensive-form games.
\newblock In {\em Proceedings of the 2013 international conference on
  Autonomous agents and multi-agent systems}, pages 271--278. International
  Foundation for Autonomous Agents and Multiagent Systems, 2013.

\bibitem{ganzfried2013action}
Sam Ganzfried and Tuomas Sandholm.
\newblock Action translation in extensive-form games with large action spaces:
  Axioms, paradoxes, and the pseudo-harmonic mapping.
\newblock In {\em Workshops at the Twenty-Seventh AAAI Conference on Artificial
  Intelligence}, 2013.

\bibitem{waugh2015solving}
Kevin Waugh, Dustin Morrill, James~Andrew Bagnell, and Michael Bowling.
\newblock Solving games with functional regret estimation.
\newblock In {\em Twenty-Ninth AAAI Conference on Artificial Intelligence},
  2015.

\bibitem{morrill2016}
Dustin Morrill.
\newblock {\em Using Regret Estimation to Solve Games Compactly}.
\newblock Master's thesis, University of Alberta, 2016.

\bibitem{deepCFR}
Noam Brown, Adam Lerer, Sam Gross, and Tuomas Sandholm.
\newblock Deep counterfactual regret minimization.
\newblock In {\em Proceedings of the 36th International Conference on Machine
  Learning (ICML-19)}, pages 793--802, 2019.

\bibitem{neuRD}
Shayegan Omidshafiei, Daniel Hennes, Dustin Morrill, Remi Munos, Julien
  Perolat, Marc Lanctot, Audrunas Gruslys, Jean-Baptiste Lespiau, and Karl
  Tuyls.
\newblock Neural replicator dynamics.
\newblock {\em arXiv preprint arXiv:1906.00190}, 2019.

\bibitem{Exp_Descent}
Edward Lockhart, Marc Lanctot, Julien P\'{e}rolat, Jean-Baptiste Lespiau,
  Dustin Morrill, Finbarr Timbers, and Karl Tuyls.
\newblock Computing approximate equilibria in sequential adversarial games by
  exploitability descent.
\newblock In {\em Proceedings of the Twenty-Eighth International Joint
  Conference on Artificial Intelligence, {IJCAI-19}}, pages 464--470.
  International Joint Conferences on Artificial Intelligence Organization, 7
  2019.

\bibitem{steinberger2019single}
Eric Steinberger.
\newblock Single deep counterfactual regret minimization.
\newblock {\em arXiv preprint arXiv:1901.07621}, 2019.

\bibitem{cesa2006prediction}
Nicolo Cesa-Bianchi and Gabor Lugosi.
\newblock {\em Prediction, learning, and games}.
\newblock Cambridge university press, 2006.

\bibitem{greenwald2006bounds}
Amy Greenwald, Zheng Li, and Casey Marks.
\newblock Bounds for regret-matching algorithms.
\newblock In {\em ISAIM}, 2006.

\bibitem{Hart00Rm}
S.~Hart and A.~Mas-Colell.
\newblock A simple adaptive procedure leading to correlated equilibrium.
\newblock {\em Econometrica}, 68(5):1127--1150, 2000.

\bibitem{freund1997decision}
Yoav Freund and Robert~E Schapire.
\newblock A decision-theoretic generalization of on-line learning and an
  application to boosting.
\newblock {\em Journal of computer and system sciences}, 55(1):119--139, 1997.

\bibitem{greenwaldtech}
Amy Greenwald, Zheng Li, and Casey Marks.
\newblock Bounds for regret-matching algorithms.
\newblock Technical Report CS-06-10, Brown University, Department of Computer
  Science, 2006.

\bibitem{gordon2005no}
Geoffrey~J Gordon.
\newblock No-regret algorithms for structured prediction problems.
\newblock Technical report, CARNEGIE-MELLON UNIV PITTSBURGH PA SCHOOL OF
  COMPUTER SCIENCE, 2005.

\bibitem{boyd2004convex}
Stephen Boyd and Lieven Vandenberghe.
\newblock {\em Convex optimization}.
\newblock Cambridge university press, 2004.

\bibitem{hazan2016introduction}
Elad Hazan.
\newblock Introduction to online convex optimization.
\newblock {\em Foundations and Trends{\textregistered} in Optimization},
  2(3-4):157--325, 2016.

\end{thebibliography}

\begin{appendices}
Below we recall results from Greenwald et. al.\cite{greenwald2006bounds} and include the detailed proofs omitted in the main body of the
paper.

Many of the following results make use of a Gordon triple. We restate the definition
from Greenwald et. al. below.

\begin{definition}
     A Gordon triple $\langle G, g, \gamma \rangle$ consists of
     three functions $G : \reals^n \mapsto \reals$, $g: \reals^n
     \mapsto \reals^n$, and $\gamma : \reals^n \mapsto \reals$
     such that for all $x,y \in \reals^n$,
     $G(x+y) \leq G(x) + g(x) \cdot y + \gamma(y)$.
\end{definition}

\section{Existing Results}
\begin{lemma}
    If $x$ is a random vector that takes values in $\reals^n$, then
    $(\expectation[\mbox{\normalfont max}_i x])^q
    \leq \expectation [\norm{x^+}^q_p]
    $
    for $p,q \geq 1$.
\end{lemma}
See [Lemma 21]\cite{greenwald2006bounds}.

\begin{lemma}
    Given a reward system $(A, \mathcal{R})$
    and a finite set of action transformations
    $\Phi \subseteq \Phi_{ALL}$, then
    $\norm{\rho^\Phi(a,r)}_p \leq U(\mu(\Phi))^{1/p}$
    for any reward function $r \in \Pi$.
\end{lemma}
The proof is indentical to [Lemma 22]\cite{greenwald2006bounds}
except we have that regrets are bounded in $[-U, U]$
instead of $[-1,1]$. Also note that by assumption
$\mathcal{R}$ is bounded.

\begin{theorem}[Gordon 2005]
Assume $\langle G, g, \gamma \rangle$ is a Gordon triple and
$C: \mathcal{N} \mapsto \reals$. Let $X_0 \in \reals^n$,
let $x_1,x_2,...$ be a sequence of random vectors over $\reals^n$,
and define $X_t = X_{t-1}+x_t$ for all times $t \geq 1$. \\
If for all times $t \geq 1$,
\[
    g(X_{t-1})\cdot \expectation[x_t|X_{t-1}] +
    \expectation[\gamma(x_t)|X_{t-1}] \leq C(t) \quad a.s.
\]
then, for all times $t \geq 0$,
\[
    \expectation[G(X_t)] \leq G(X_0) + \sum_{\tau=1}^t{C(\tau)}.
\]
\end{theorem}
It should be noted that the above theorem was originally proved
by Gordon \cite{gordon2005no}.

\section{Proofs}

\epsbound*
\begin{proof}
We denote $r = (r^\prime(a))_{a \in A}$ as the reward vector for an arbitrary reward function
$r^\prime : A \mapsto \reals$. Since by construction this algorithm chooses $L_t$ at each timestep $t$ to be the fixed point of $\Tilde{M}_t$, all that remains to be shown is that this algorithm satisfies the $(\Phi, f, \epsilon)$-Blackwell condition with
$\epsilon \leq 2 \supReward \norm{ \ylink[\Phi]{t} - \ylinkEst[\Phi]{t} }_1, t > 0$.

By expanding the value of interest in the $(\Phi,f)$-Blackwell condition and applying elementary upper bounds, we arrive at the desired bound. For simplicity, we omit timestep indices and set $\algSimple := \alg$.
First, suppose $\sum_{\phi \in \Phi}\ylinkEst{t} \neq 0$:
\begin{align*}
\ylinkSimple &\cdot \mathbb{E}_{a \sim \algSimple}[\rho^\Phi (a,r)]
= \sum_{\phi \in \Phi}{
  \ylinkSimple[\phi](r\cdot [\phi](\algSimple) - r \cdot \algSimple)
} \\
&= r \cdot \left(
  \sum_{\phi \in \Phi} \ylinkSimple[\phi] [\phi]\algSimple - \algSimple
\right). \text{ By adding and subtracting $\ylinkEstSimple$,}\\
&= r \cdot \left(
\sum_{\phi \in \Phi} (
  \ylinkEstSimple[\phi] - \ylinkEstSimple[\phi] +
  \ylinkSimple[\phi]
)(
  [\phi](\algSimple) - \algSimple
)
\right). \text{ By expanding, as well as multiplying and dividing by $(\sum_{\phi \in \Phi} \ylinkEstSimple[\phi])$,}\\
&=r \cdot \left(
  (\sum_{\phi \in \Phi} \ylinkEstSimple[\phi])
  \sum_{\phi \in \Phi} \Tilde{M} \algSimple - \algSimple +
  \sum_{\phi \in \Phi} (
    \ylinkSimple[\phi] - \ylinkEstSimple[\phi]
  )(
    [\phi](\algSimple) - \algSimple
  )
\right). \text{ Since $\algSimple$ is a fixed point of $\Tilde{M}$,}\\
&=r \cdot \left(
  \sum_{\phi \in \Phi} (
    \ylinkSimple[\phi] - \ylinkEstSimple[\phi]
  )(
    [\phi](\algSimple) - \algSimple
  )
\right). \text{ By the generalized Cauchy-Schwarz inequality~\cite{boyd2004convex, hazan2016introduction},}\\
&\leq \norm{r}_\infty \norm{
     \sum_{\phi \in \Phi} (
      \ylinkSimple[\phi] - \ylinkEstSimple[\phi]
    )(
      [\phi](\algSimple) - \algSimple
    )
}_1 \\
&\leq \norm{r}_\infty
  \sum_{\phi \in \Phi} | \ylinkSimple[\phi] - \ylinkEstSimple[\phi] | (
    \norm{[\phi](\algSimple)}_1 +
    \norm{\algSimple}_1
  ). \text{ Since $[\phi](\algSimple)$ and $\algSimple$ are both distributions,}\\
&\leq \norm{r}_\infty
  \sum_{\phi \in \Phi} | \ylinkSimple[\phi] - \ylinkEstSimple[\phi] |(1 + 1) \\
&\leq 2 \supReward
  \norm{ \ylinkSimple[\Phi] - \ylinkEstSimple[\Phi] }_1.
\end{align*}
If $\sum_{\phi \in \Phi}\ylinkEstSimple = 0$ it is easy to see the inequality still holds.

Therefore, $\{L_t\}_{t=1}^\infty$ satisfies the $(\Phi, f, \epsilon)$-Blackwell condition with
$\epsilon \leq 2 \supReward \norm{ \ylinkSimple[\Phi] - \ylinkEstSimple[\Phi] }_1$, as required to complete the argument.

\end{proof}
An important observation of Theorem \ref{thm-epsbound} is the following corollary:
\begin{corollary}
For a reward system $(A,\mathcal{R})$, finite set of
action transformations $\Phi \subseteq \Phi_{ALL}$, and two link functions $f$ and $f^\prime$,
if there exists a strictly positive function
$\psi : \reals^{|\Phi|}\mapsto \reals$ such that $f^\prime(x) = \psi(x)f(x)$
then for any $\epsilon \in \reals$,
then an approximate $(\Phi, f)$-regret-matching algorithm satisfies
\[
    f^\prime(R^\Phi_{t-1}(h)) \cdot \expectation_{a \sim \alg}[\rho^\Phi(a,r)] \leq 2 \supReward
                 \norm{f^\prime(R^\Phi_{t-1})-f^\prime(\Tilde{R}^\Phi_{t-1})}_1.
\]
\end{corollary}

\begin{proof}
    The reasoning is similar to [Lemma 20]\cite{greenwald2006bounds}. The played fixed point is the
    same under both link functions, thus following the same
    steps to Theorem 1 provides the above bound.
\end{proof}

\expectationbound*

\begin{proof}
    The proof is similar to [Corollary
    7]\cite{greenwald2006bounds} except that the learning
    algorithm is playing the approximate fixed point with respect to the link function $g$.
    From Theorem 1 we have $g(R^\Phi_{t-1}(h))\cdot \expectation_{a \sim \alg}[\rho^\Phi(a,r)] \leq 2U\norm{g(R^\Phi_{t-1})-g(\Tilde{R}^\Phi_{t-1})}_1$.
    Noticing that $\expectation_{a \sim \alg}[\rho^\Phi(a,r)] =
    \expectation[\rho^\Phi(a,r)|R^\Phi_{t-1}]$ and taking
    $x_t = \rho^\Phi(a,r), X_t=R^\Phi_t$ we have
    \[
      g(X_{t-1})\cdot \expectation[x_t|X_{t-1}] +
    \expectation[\gamma(x_t)|X_{t-1}] \leq 2U\norm{g(R^\Phi_{t-1})-g(\Tilde{R}^\Phi_{t-1})}_1 +
    \underset{a \in A, r \in \Pi}{\sup} \gamma(\rho^{\Phi}(a,r)).
    \]
The result directly follows from Theorem 6 by taking $C(\tau) = 2U\norm{g(R^\Phi_{\tau-1})-g(\Tilde{R}^\Phi_{\tau-1})}_1 +
    \underset{a \in A, r \in \Pi}{\sup} \gamma(\rho^{\Phi}(a,r))$.
\end{proof}

\polyONE*

\begin{proof}
The proof follows closely to [Theorem 9]\cite{greenwald2006bounds}.
Taking $G(x) = \norm{x^+}^2_p$ and $\gamma(x)=(p-1)\norm{x}^2_p$ then
$\langle G, g, \gamma \rangle$ is a Gordon triple \cite{greenwald2006bounds}.
Given the above gordon triple we have
\begin{align}
       \left( \expectation\left[
        \underset{\phi \in \Phi}{\mbox{max}}R^\phi_t
    \right] \right)^2 &\leq \expectation[\norm{(R^\Phi_t)^+}]^2_p\\
    &= \expectation[G(R^\Phi_t)] \\
    &\leq G(0) +
    t \underset{a \in A, r \in \Pi}{\sup} \gamma(\rho^{\Phi}(a,r)) +
        2\supReward \sum_{s=1}^t{\norm{g(R^\Phi_{s-1})-g(\Tilde{R}^\Phi_{s-1})}_1} \\
    &\leq
    G(0)  + t(p-1)U^2(\mu(\Phi))^{2/p}
    + 2 U \sum_{k=1}^t \norm{g(R^\Phi_{k-1}) - g(\Tilde{R}^\Phi_{k-1})}_1
\end{align}
The first inequality is from Lemma 1. The second inequality follows
from Corollary 1 and theorem 2. The third inequality is an application of Lemma 2. The result
then immediately follows.
\end{proof}

\polyTWO*

\begin{proof}
The proof follows closely to [Theorem 11]\cite{greenwald2006bounds}.
Taking $G(x) = \norm{x^+}^p_p$ and $\gamma(x)=(p-1)\norm{x}^p_p$ then
$\langle G, g, \gamma \rangle$ is a Gordon triple \cite{greenwald2006bounds}.
Given the above Gordon triple we have
\begin{align}
       \left( \expectation\left[
        \underset{\phi \in \Phi}{\mbox{max}}R^\phi_t
    \right] \right)^p &\leq \expectation[\norm{(R^\Phi_t)^+}]^p_p\\
    &= \expectation[G(R^\Phi_t)] \\
    &\leq G(0) +
    t \underset{a \in A, r \in \Pi}{\sup} \gamma(\rho^{\Phi}(a,r)) +
        2\supReward \sum_{s=1}^t{\norm{g(R^\Phi_{s-1})-g(\Tilde{R}^\Phi_{s-1})}_1} \\
    &\leq
    G(0)  + t U^p(\mu(\Phi))
    + 2 U \sum_{k=1}^t \norm{g(R^\Phi_{k-1}) - g(\Tilde{R}^\Phi_{k-1})}_1
\end{align}
The first inequality is from Lemma 1. The second inequality follows
from Corollary 1 and theorem 2. The third inequality is an application of Lemma 2. The result
then immediately follows.
\end{proof}

\exponential*

\begin{proof}
The proof follows closely to [Theorem 13]\cite{greenwald2006bounds}.
Taking $G(x) = \frac{1}{\eta}\mbox{ln}\left(\sum_i{e^{\eta x_i}} \right)$ and $\gamma(x)=\frac{\eta}{2}\norm{x}^2_\infty$ then
$\langle G, g, \gamma \rangle$ is a Gordon triple \cite{greenwald2006bounds}.
Given the above Gordon triple we have
\begin{align}
       \expectation\left[
        \underset{\phi \in \Phi}{\mbox{max }}\eta R^\phi_t
    \right] & = \expectation \left[
         \mbox{ln }e^{\underset{\phi \in \Phi}{\mbox{max }} \eta R^\phi_t}
    \right]\\
    &= \expectation \left[
         \mbox{ln } \underset{\phi \in \Phi}{\mbox{max }} e^{ \eta R^\phi_t}
    \right]\\
    &\leq \expectation \left[
         \mbox{ln } \sum_{\phi \in \Phi} e^{ \eta R^\phi_t}
    \right]\\
    &= \eta \expectation[G(R^\Phi_t)] \\
    &\leq \eta \left( G(0) +
    t \underset{a \in A, r \in \Pi}{\sup} \gamma(\rho^{\Phi}(a,r)) +
        2\supReward \sum_{s=1}^t{\norm{g(R^\Phi_{s-1})-g(\Tilde{R}^\Phi_{s-1})}_1} \right) \\
    &\leq\eta \left(G(0) +
    t\frac{\eta}{2}U^2 +
        2\supReward \sum_{s=1}^t{\norm{g(R^\Phi_{s-1})-g(\Tilde{R}^\Phi_{s-1})}_1} \right)
\end{align}
The second inequality follows
from Corollary 1 and theorem 2. The result
then immediately follows.
\end{proof}

\end{appendices}

\end{document}